\documentclass{article}

\usepackage[square,sort,comma,numbers]{natbib}

\usepackage[preprint]{neurips_2019}

\usepackage[utf8]{inputenc} 
\usepackage[T1]{fontenc}    
\usepackage{hyperref}       
\usepackage{url}            
\usepackage{booktabs}       
\usepackage{amsfonts}       
\usepackage{nicefrac}       
\usepackage{microtype}      

\usepackage{hyperref}
\usepackage{url}
\usepackage{wrapfig}

\usepackage{times}
\usepackage{helvet}
\usepackage{courier}
\usepackage{amssymb}
\usepackage{amsmath}
\usepackage{graphicx}
\usepackage{amsfonts}       
\usepackage{xcolor}
\usepackage{subcaption}
\usepackage{pifont}
\usepackage{mathtools}
\usepackage{tikz}
\usepackage{bm}
\usepackage{multicol}
\usepackage{multirow}
\usepackage{wrapfig}
\usepackage{algorithm}
\usepackage{algpseudocode}
\usepackage[normalem]{ulem}
\usepackage{amsmath,amsfonts,amssymb,amsthm,array}

\DeclareMathOperator*{\argmin}{arg\,min}
\newtheorem{thm}{Theorem}[section]
\newtheorem{lem}{Lemma}
\newtheoremstyle{propstyle}
  {\topsep} 
  {\topsep} 
  {} 
  {} 
  {\bfseries} 
  {.} 
  {.5em} 
  {} 

\theoremstyle{plain}

\title{Layer-wise Learning of Stochastic Neural Networks with Information Bottleneck}
\author{
 Thanh T. Nguyen \\
  School of Computer Science and Engineering\\
Ulsan National Institute of Science and Technology\\
Ulsan, Republic of Korea 44919 \\
\texttt{thanhnt@unist.ac.kr} \\
  \And 
  Jaesik Choi  \\
    School of Computer Science and Engineering\\
Ulsan National Institute of Science and Technology\\
Ulsan, Republic of Korea 44919 \\
\texttt{jaesik@unist.ac.kr} \\
}

\begin{document}

\maketitle

\begin{abstract}
Information Bottleneck (IB) is a generalization of rate distortion theory that naturally incorporates compression and relevance trade-offs for learning. Though the original IB has been extensively studied, there has not been much understanding about multiple bottlenecks which better fit in the context of neural networks. In this work, we propose Information Multi-Bottlenecks (IMBs)
 as an extension of IB to multiple bottlenecks which has a direct application to training neural networks by considering layers as multiple bottlenecks and weights as parameterized encoders and decoders. We show that the multiple optimality of IMB are not simultaneously achievable for stochastic encoders. We thus propose a simple compromised scheme of IMB which in turn generalizes maximum likelihood estimate (MLE) principle in the context of stochastic neural networks. We demonstrate the effectiveness of IMB on classification tasks and adversarial robustness in MNIST and CIFAR10. 

\end{abstract}

\section{Introduction}
\label{sec:intro}
The Information Bottleneck (IB) principle \cite{Tishby99} 
extracts relevant information about a target variable $Y$ from an input variable $X$ via a \textit{single} bottleneck variable $Z$. In detail, the IB framework constructs a \textit{bottleneck} variable $Z = Z(X)$ that is a \textit{compressed} version of $X$ but preserves as much \textit{relevant} information in $X$ about $Y$ as possible. The compression of the representation $Z$ is quantized by $I(Z;X)$, the mutual information of $Z$ and $X$. The relevance in $Z$, the amount of information $Z$ contains about $Y$, is specified by $I(Z;Y)$. An optimal representation $Z$ satisfying a certain compression-relevance trade-off constraint is then determined via minimization of the following Lagrangian $\mathcal{L}_{IB}[p(z|x)] = I(Z;X) - \beta I(Z;Y)$,
where $\beta$ is a positive Lagrangian multiplier that controls the trade-off.


Deep neural networks (DNNs) have demonstrated state-of-the-art performances in several important machine learning tasks including image recognition \cite{DBLP:conf/nips/KrizhevskySH12}, natural language translation \cite{cho2014learning,bahdanau2014neural} and game playing \cite{DBLP:journals/nature/SilverHMGSDSAPL16}. Behind the practical success of DNNs there are various revolutionary techniques such as 
data-specific design of network architecture (e.g., convolutional neural network architecture), regularization techniques (e.g., early stopping, weight decay, dropout \cite{DBLP:journals/jmlr/SrivastavaHKSS14}, and batch normalization \cite{DBLP:conf/icml/IoffeS15}), and optimization methods \cite{DBLP:journals/corr/KingmaB14}. For learning DNNs, the maximum likelihood estimate (MLE) principle (in its various forms such as maximum log-likelihood or Kullback-Leibler divergence) has generally become a de-facto standard. The MLE principle maximizes the likelihood of the model for observing the entire training data. This principle is, however, generic and not specially tailored to hierarchy-structured models like neural networks. Particularly, MLE treats the entire neural network as a collective body without considering an explicit contribution of its hidden layers to model learning. As a result, the information contained within the hidden structure may not be adequately modified to capture the data regularities reflecting a target variable. Thus, a reasonable question to ask is whether the MLE principle effectively and sufficiently exploits a neural network's representative power and whether there is any better alternative?

In this work, we propose a unifying perspective to bridge between IB and neural networks via Information Multi-Bottlenecks (IMBs). The core idea is that we extend the original IB to multiple bottlenecks whereby neural networks can be viewed as a parameterized version of IMBs. We show a conflicting optimality for multiple bottlenecks under some mild conditions which are the case for stochastic neural networks. When applied to stochastic neural networks, IMBs readily enjoy various mutual information approximations (such as Variational Information Bottleneck \cite{DBLP:journals/corr/AlemiFD016}) and gradient estimation. Consequently, we show how IMBs provide an alternative learning principle for stochastic neural networks as compared to the standard MLE principle. Finally, we demonstrate the superior or at least competitive empirical performance of IMBs, compared to MLE, in terms of generalization, and adversarial robustness for stochastic neural networks on MNIST and CIFAR10. We also show that IMBs empirically learn the neural representations better in terms of information exploitation.

This paper is organized as follows. We first review related literature in Section \ref{sec:related_work}. Section \ref{sec:pib} introduces our Information Multi-Bottlenecks with some important insights and how it is applied to stochastic neural networks. Section \ref{sec:app} demonstrates a case study that we apply our IMB to binary stochastic neural networks. Finally, Section \ref{sec:exp} presents the empirical results of our framework on MNIST and CIFAR10, in comparison with MLE and Variational Information Bottleneck.

\section{Related Work}
\label{sec:related_work}
Our IMBs are a direct multi-bottleneck extension of Information Bottleneck (IB) proposed in \cite{Tishby99}. IB generalizes the rate distortion theory to better fit in the scenario of learning. IB extracts the relevant information in one variable $X$ about another variable $Y$ via some intermediate variable $Z$. This IB problem has been solved efficiently in the following three cases only: (1)  $X,Y$ and $Z$ are all discrete \cite{Tishby99}; (2) $X,Y$ and $Z$ are mutually joint Gaussian \cite{DBLP:journals/jmlr/ChechikGTW05}; (3) or $(X,Y,Z)$ has meta-Gaussian distributions \cite{DBLP:conf/nips/ReyR12}. However, the extension of the IB principle to multiple bottlenecks is not straightforward. In addition, the IB principle has been proven to be mathematically equivalent to the MLE principle for the multinomial mixture model for the clustering problem when the input distribution $X$ is uniform or has a large sample size \cite{Slonim03}. It is, however, not clear how the IB principle is related to the MLE principle in the context of DNNs. In IMBs, we extend IB to multiple bottlenecks and show the connection between IMBs and MLE in stochastic neural networks. 

Our work also shares with the literature in training DNNs. Perhaps the most common way to generalize the MLE principle in DNNs is to apply Bayesian modeling to the weights of DNNs \cite{Neal:1995:BLN:922680,DBLP:journals/neco/MacKay92a,DBLP:journals/neco/DayanHNZ95}. This approach equips DNNs with uncertainty reasoning and takes good advantage of the well-studied tools in probability theory. As a result, this idea achieved interpretability and state-of-the-art performance on many tasks \cite{Gal2016Uncertainty}. The main challenges of this approach is to approximate the intractable posterior and scale the model to high-dimensional data. IMBs, on the other hand, take a different but very natural perspective on DNNs by reasoning about the learning in terms of information contained in the neural representations. 

Some works have considered applying Information Bottleneck to multiple layers in neural networks. For example, \cite{DBLP:conf/itw/TishbyZ15} proposes the use of the mutual information of a hidden layer with the input layer and the output layer to quantify the performance of DNNs. A different family of work along the line approximates the mutual information of high-dimensional variables arose from DNNs \cite{DBLP:conf/uai/StrouseS16,DBLP:conf/nips/ChalkMT16,DBLP:journals/corr/AlemiFD016}. While we also apply it to neural networks, the key difference in our work is that we explicitly develop a framework of Information Multi-Bottlenecks which are specifically calibrated to each layers of neural networks. 
 
A possibly (but not closely) related work is the idea of layer-wise learning in IMBs in which certain aspects of each layer are considered to train a neural network. A great deal of work along this line is greedy layer-wise training of DNNs \cite{Hinton2006,Bengio2006}. While this line of work uses unsupervised greedy pretraining of DNNs, IMB is not necessarily a pretraining method nor a greedy algorithm. Especially, the unsupervised greedy pretraining of stacked autoencoder is proven to be equivalent to maximizing mutual information between the input variable and latent variable \cite{Vincent2008}. In contrast, IMBs follow a different principle which aims at preserving the relevant information under the compression at the same time.

\section{Information Multi-Bottlenecks}
\label{sec:pib}
\paragraph{Notations} We denote random variables (RVs) by capital letters, e.g., $X$, and its specific realization value by the corresponding little letter, e.g., $x$. Note that $x$ can be vector-valued.  We write $Y \rightarrow X \rightarrow Z$ to denote a Markov chain where $Y$ and $Z$ are independent given $X$. We write $X \perp Y$ (respectively, $X \not \perp Y$) to indicate $X$ and $Y$ are independent (respectively, not independent) and abuse the notation integral, e.g.,  $\int f(z) d z$, regardless of whether the variable $ {z}$ is real-valued or discrete-valued.

\subsection{Information Multi-Bottlenecks}
Information Bottleneck (\cite{Tishby99}) extends the notion of rate distortion theory to incorporate compression and relevance trade-offs which are more suitable for learning than rate distortion theory. The information about data variables $X$ and $Y$ (e.g., $X$ represents MNIST images and $Y$ represents the MNIST labels) is squeezed through a bottleneck (possibly multivariate) variable $Z$. The goal is to find an optimal encoder $p(z|x)$ such that $Z$ preserves only relevant information in $X$ about $Y$ and compress irrelevant one. This goal can be formally described in a constrained optimization in terms of mutual information. 

Here we consider a natural extension of IB to multiple bottlenecks $Z_1, Z_2, ..., Z_L$ such that the bottlenecks form a Markov chain. The extension simply introduces multiple constrained Information Bottleneck optimization problems for each of the bottlenecks:
\begin{align*}
    \min_{p( {z}_l| {x})} \mathcal{L}_l[p( {z}_l| {x})] := \min_{p( {z}_l| {x})} \{ I(Z_l; X) - \beta_l I(Z_l; Y) \}, 1 \leq l \leq L. 
\end{align*}
Despite being simple, this multi-objective optimization is very challenging especially when the bottlenecks are high-dimensional. This extension is also essential in the context of neural networks as it solves an important problem that otherwise it is not possible in the conventional IB. The approach that uses the conventional IB to neural networks (e.g., \cite{DBLP:journals/corr/AlemiFD016} directly apply IB to neural networks by considering the entire neural network as a parameterized encoder in IB) does not fully leverage the multi-layered structure of neural networks. Each hidden layer has an important contribution to the compression and relevance process in the neural networks; thus should be explicitly considered in the learning process via information bottleneck principle. Therefore, multiple bottlenecks are more suitable in neural networks, and thus it is important to understand multiple bottlenecks. In the following theorem, we identify when it is possible to obtain a non-conflicting optimal solutions for multiple bottlenecks. 
\begin{thm}[\textit{Conflicting Multi-Information Optimality}]
\label{thm:conflicting_optimality}
Given four random variables $X, Y, Z_1,$ and $Z_2$ such that $Y \rightarrow X \rightarrow Z_1 \rightarrow Z_2$ and $H(Y|X) > 0$ and two constrained minimization problems: 
\begin{align}
\label{eq:MO_lagrange}
\argmin_{p( {z}_l| {x})} \mathcal{L}_l[p( {z}_l| {x})] := \argmin_{p( {z}_l| {x})} \{ I(Z_l; X) - \beta_l I(Z_l; Y) \}, l \in \{1,2\}
\end{align}
where $0 < \beta_1 < \infty$, $0 < \beta_2 < \infty$, and $p(z_2|x) = \int p(z_2 | z_1) p(z_1 |x) dz_1$. Then, the following two statements are equivalent: 
\begin{enumerate}
    \item $Z_1$ and $Z_2$ do not satisfy either of the following conditions: 
    \begin{enumerate}
        \item $Z_2$ is a \textit{sufficient statistics} of $Z_1$ for $X$ and $Y$ (i.e., $Y \rightarrow X \rightarrow Z_2 \rightarrow Z_1$), 
        \item $Z_2$ is independent of $Z_1$. 
    \end{enumerate}
    
    \item The two constrained minimization problems defined above are \underline{conflicting}, i.e., there does not exist a single solution that minimizes $\mathcal{L}_1$ and $\mathcal{L}_2$ simultaneously.
\end{enumerate}
\end{thm}
\begin{proof}[\textit{Sketch proof}]
 Leveraging the Markovian structure of four random variables in the two Information Bottleneck objectives and using Lagrangian multipliers. The detailed proof is in Appendix A (the supplementary). 
\end{proof}

Theorem \ref{thm:conflicting_optimality} suggests that the multi-information optimality (the optimal conditions for each of the individual Information Bottleneck optimization problems in IMBs) conflicts for most cases of interest, e.g. stochastic neural networks which we will present in detail in next subsection. The values of $\beta_1$ and $\beta_2$ are important to control the number of bits we extract the relevant information into the bottlenecks and determine the conflictability of multiple bottlenecks on the edge cases. Recall that for $Y \rightarrow X \rightarrow X$, we have $0 \leq I(Z;X) \leq H(X); 0 \leq I(Z;Y) \leq I(X;Y)$. If $\beta_1$ and $\beta_2$ go to infinity, the optimal bottlenecks $Z_1$ and $Z_2$ are both deterministic function of $X$ and they do not conflict. When $\beta_1 = \beta_2 = 0$, the information about $Y$ in $X$ is maximally compressed in $Z_1$ and $Z_2$ (i.e., $Z_1 \perp X, Z_2 \perp X$), and they do not conflict. But the optimal solutions conflict when $\beta_1 = 0$ and $\beta_2 > 0$ as the former leads to a maximally compressed $Z_1$ while the latter prefers an informative $Z_2$ (this contradicts the Markov structure $X \rightarrow Z_1 \rightarrow Z_2$ which indicates that maximal compression in $Z_1$ leads to maximal compression in $Z_2$). We can also easily construct some non-conflicting IMBs for $0 \leq \beta_1, \beta_2 \leq \infty$ that violates the conditions. For example, if $X$ and $Y$ are jointly Gaussian, the optimal bottlenecks $Z_1$ and $Z_2$ are linear transform of $X$ and jointly Gaussian with $X$ and $Y$ (\cite{DBLP:journals/jmlr/ChechikGTW05}). In this case, $Z_2$ is a sufficient statistic of $Z_1$ for $X$. In the case of neural networks, we can also construct a simple but non-trivial neural network that can obtain a non-conflicting multi-information optimality. For example, consider a neural network of two hidden layers $Z_1$ and $Z_2$ where $Z_1$ is arbitrarily mapped from the input layer $X$ but $Z_2$ is a sample mean of $n$ samples i.i.d. drawn from the normal distribution $\mathcal{N}(Z_1; \Sigma)$. This construction guarantees that $Z_2$ is a sufficient statistic of $Z_1$ for $X$, thus there is non-conflicting multi-information optimality.

\subsection{Stochastic Neural Networks}
Stochastic neural networks has been studied in literature \cite{DBLP:conf/nips/TangS13, DBLP:journals/corr/RaikoBAD14, DBLP:journals/access/ShafieeSW16}. One important advantage of stochastic neural networks is that they can induce rich multi-model distributions in the output space \cite{DBLP:conf/nips/TangS13} and enable exploration in reinforcement learning \cite{DBLP:conf/iclr/FlorensaDA17}. 
Here we consider a stochastic neural network with $L$ hidden layers without any feedback or skip connection, we view the input layer $X$, the output of the $l^{th}$ hidden layer $Z_l$, and the network output layer $\hat{Y}$ as random variables (RVs). Without any feedback or skip connection, $Y, X, Z_l, Z_{l+1},$ and $\hat{Y}$ form a Markov chain in that order, denoted as: 
\begin{align}
\label{eq:markov_chain}
Y \rightarrow X \rightarrow Z_l \rightarrow Z_{l+1} \rightarrow \hat{Y}
\end{align}


The role of the neural network is, therefore, reduced to transforming from one RV to another via the Markov chain $X \rightarrow Z_l \rightarrow Z_{l+1} \rightarrow \hat{Y}$ where $\hat{Y}$ is used as a surrogate for $Y$. We call the transition distribution $p(Z_l |X)$\footnote{If the mapping from $X$ to $Z_l$ is deterministic, then $p(Z_l |X)$ is simply a delta function.} from $X$ to $Z_l$ an \textit{encoder} as it encodes the data $X$ into the representation $Z_l$. For each encoder $p(Z_l |X)$, there is a unique corresponding decoder, namely \textit{relevance decoder}, that decodes the relevant information in $X$ about $Y$ from representation $Z_l$:
\begin{align}
\label{eq:relevance_decoder}
p( {y}| {z}_l) = \int p_D( {x}, {y}) \frac{p( {z}_l| {x})}{p( {z}_l)} d {x}
\end{align}
It follows from the Markov chain in Equation (\ref{eq:markov_chain}) that inference in a neural network can be done as: 
 \begin{align}
p(\hat{ {y}}| {x}) &= \int p(\hat{ {y}},  {z}|  {x}) d {z}  = \int p(\hat{ {y}}| {z}) p( {z}| {x}) d {z} = \int \prod_{l=1}^{L+1} p( {z}_l |  {z}_{l-1}) d {z}
\end{align} 
where $ {z} = ( {z}_1, ...,  {z}_L)$, $Z_0 := X$ and $Z_{L+1} := \hat{Y}$. 



\subsection{Stochastic Neural Networks as Information Muti-Bottlenecks}
We consider a stochastic neural network as a parameterized version of Information Multi-Bottlenecks in which each layer is a bottleneck and the weights connecting the layers are parameterized encoders and relevance decoders. Specifically, $p(Z_l| X)$ is parameterized by the sub-network from the input layer to layer $l$. In this perspective, a stochastic neural network is a data-processing system that transforms a data distribution via a series of bottlenecks $Z_l$. Thus, the role of each layer can be interpreted as information filter that compress irrelevant information and preserve relevant one. This notion of compression and relevance can be captured with mutual information $I(Z_l; X)$ and $I(Z_l; Y)$, respectively. Subsequently, it is natural to interpret the learning of a stochastic neural networks as a multi-objective optimization: 
\begin{align}
\label{eq:MO_lagrange}
\argmin_{p(z_l|x)} \mathcal{L}_l[p( {z}_l| {x})] := \argmin_{p(z_l|x)} \{ I(Z_l; X) - \beta_l I(Z_l; Y)\}, 1 \leq l \leq L.
\end{align}
where  $\beta_l$ are the positive Lagrange multipliers for the constraints. 

We first present how to approximate mutual information $I(Z_l; X)$ and $I(Z_l; Y)$ using variational mutual information. After that, we present how to solve the multi-objective optimization in Eq. \ref{eq:MO_lagrange}. 
\subsubsection{Approximate Relevance} The relevance $I(Z_l;Y)$ is intractable due to the intractable relevance decoder $p( {y}| {z}_l)$ in Equation (\ref{eq:relevance_decoder}). It follows from the non-negativity of Kullback-Leibler divergence that:
\begin{align}
H(Y|Z_l) &= -\int p( {y}| {z}_l) p( {z}_l) \log p( {y}| {z}_l) d {y}d  {z}_l  \leq -\int p( {y}| {z}_l) p( {z}_l) \log p_{v}( {y}| {z}_l) d {y} d {z}_l \nonumber \\
\label{variational_relevance}
& \!\!\! = -\mathbb{E}_{(X,Y)_D}  \mathbb{E}_{Z_l|X} \log p_{v}(Y|Z_l)  =: \tilde{H}(Y|Z_l) 
\end{align}
where $p_{v}( {y}| {z}_l)$ is any probability distribution. Note that $I(Z_l; Y) = H(Y) - H(Y|Z_l)$ where $H(Y) = constant$ which can be ignored in the minimization of $\mathcal{L}_l$. Specifically in IMB, we propose to use the network architecture connecting $Z_l$ to $\hat{Y}$ to define the variational relevance decoder for layer $l$, i.e., $p_{v}( {y}| {z}_l) = p( {\hat{y}}| {z}_l)$ where $p( {\hat{y}}| {z}_l)$ is determined by the network architecture:
\begin{align}
\label{eq:stochastic_prediction_at_l}
p_{v}( {y}| {z}_l) &:= p( {\hat{y}}| {z}_l) = \int \prod_{i=l}^{L} p( {z}_{i+1} |  {z}_i) d {z}_{L} ... d {z}_{l+1} = \mathbb{E}_{Z_L |  {z}_l} \left[ p( {\hat{y}}|Z_L) \right].
\end{align}

For the rest of this work, we refer to $\tilde{H}(Y|Z_l)$ with $p_{v}( {y}| {z}_l) = p( {\hat{y}}| {z}_l)$ as the \textit{variational conditional relevance} (VCR) of the $l^{th}$ layer. Theorem \ref{prop:1} addresses the relation between VCR and the MLE principle.  
\begin{thm}[\textit{Information on the extreme layers}]
The VCR of the lowest-level (so-called \textit{\textbf{super}}) layer (i.e., $l=0$) is the negative log-likelihood (NLL) function of the neural network, i.e.,  
\begin{align}
    \tilde{H}(Y | Z_0) = -\mathbb{E}_{(X, Y)_D} \left[ \log p(\hat{Y}| X) \right]. 
\end{align}
Similarly, the VCR of the highest-level layer (i.e., $l=L$) equals that of the \textit{\textbf{compositional}} layer $Z = (Z_1, Z_2, ..., Z_L)$, a composite of all hidden layers; in addition, their VCR is an upper bound on the NLL:
\begin{align}
    \tilde{H}(Y|Z_L) = \tilde{H}(Y|Z) \geq - \mathbb{E}_{(X,Y)_D} \left[ \log p(\hat{Y}|X) \right].
\end{align}
\label{prop:1}
\end{thm}
\begin{proof}[Sketch Proof]
Using the definition of VCR, the Markov Chain assumption, and Jensen's inequality. The detailed proof is in Appendix B (the supplementary). The details of how to decompose VCR for multivariate variable $Y$ can also be found in Appendix F (the supplementary). 
\end{proof}

An interpretation of MLE in terms of VCR and vice versus is immediately followed from Theorem \ref{prop:1}. That said, the MLE principle is to optimize the super-level VCR while VCR allows an explicit extension of this concept to any layer.
\subsubsection{Approximate Compression}
While $p( {z}_l |  {z}_{l-1})$ in DNNs has an analytical form, $p( {z}_l |  {x})$ for $l > 1$ generally does not as it is a mixture of $p( {z}_l |  {x})$. We thus propose to avoid directly estimating $I(Z_l; X)$ by instead resorting to its upper bound $I(Z_l; Z_{l-1})$ as its surrogate in the optimization.  However, $I(Z_l; Z_{l-1})$ is still intractable as it has the intractable distribution $p( {z}_l)$. We then approximate $I(Z_l; Z_{l-1})$ using a mean-field (factorized) variational distribution $ r( {z}_l) = \prod_{i=1}^{n_l} r(z_{l,i})$:
\begin{align}
\label{eq:comp_est}
I(Z_l; X) \leq I(Z_l; Z_{l-1}) \leq \mathbb{E}_{Z_{l-1}} \sum_{i=1}^{n_l} D_{KL} \left[ p(Z_{l,i} | Z_{l-1}) || r(Z_{l,i}) \right] =: \tilde{I}(Z_l; Z_{l-1}).
\end{align}
The detailed derivations can be found in the Appendix C (the supplementary). 

\subsection{Compromised Information Optimality}
Due to Theorem \ref{thm:conflicting_optimality}, we cannot achieve the information optimality for simultaneously all layers; thus we need some compromised approach to instead obtain a compromised optimality. We propose two natural compromised strategies, namely \texttt{JointIMB} and \texttt{GreedyIMB}. \texttt{JointIMB} (Algorithm \ref{alg:JointPIB}) is a weighted sum of the variation IB objectives $\mathcal{L}^{joint} := \sum_{l=0}^L \gamma_l \mathcal{\tilde{L}}_l$ where $\mathcal{\tilde{L}}_l$ is the variational approximation of $\mathcal{L}_l$ using approximate relevance (Eq. (\ref{variational_relevance})) and approximate compression (Eq. (\ref{eq:comp_est})), and $\gamma_l \geq 0$. The main idea of \texttt{JointMIB} is to simultaneously optimize all encoders and variational relevance decoders. Even though each layer might not achieve its individual optimality, their joint optimality encourages a joint compromise. On the other hand, \texttt{GreedyIMB} applies PIB progressively in a greedy manner. In other words, \texttt{GreedyIMB} tries to obtain the conditional optimality of a current layer which is conditioned on the achieved conditional optimality of the previous layers. 

\begin{wrapfigure}{R}{0.57\textwidth}
\begin{minipage}{0.57\textwidth}
\begin{algorithm}[H]
\caption{\texttt{JointIMB}}
\label{alg:JointPIB}
\begin{algorithmic}[1]
\Procedure{\texttt{JointIMB}}{}
\State \textbf{Input}: $S_0 \leftarrow ( {x}_i,  {y}_i)_{i=1}^{N} \sim p_D( {x},  {y}), \gamma_l, \beta_l$
\State \textbf{Initialization}:  {$\theta$}
\While {not converged}
\For {$i = 1$ to $L$}
\State $S_i = \emptyset$
\For {$z_{i-1} \in S_{i-1}$}
\State $S_i \leftarrow S_i \cup \{ ( {z}^{(k)}_i)_{k=1}^M \sim p( {z}_i |  {z}_{i-1})$\} 
\EndFor
\EndFor
\State $\mathcal{\tilde{L}}^{joint}( {\theta}) \leftarrow$ Eqs (\ref{variational_relevance},\ref{eq:stochastic_prediction_at_l},\ref{eq:comp_est}), and $\{S_i\}_{i=0}^{L}$
\State $ {g} \leftarrow \displaystyle \frac{\partial }{\partial  {\theta}} \mathcal{\tilde{L}}^{joint}( {\theta})$ 
\State  $ {\theta} \leftarrow  {\theta} - \nu  {g}$
\EndWhile
\State \textbf{Output}:  {$\theta$}
\EndProcedure
\end{algorithmic}
\end{algorithm}
\end{minipage}
 \end{wrapfigure}

\section{A Case Study: Binary Stochastic Neural Networks}
\label{sec:app}

To analyze IMB, we apply it to a simple network architecture: binary stochastic feed-forward (fully-connected) neural networks (SFNN) though the extension to real-valued stochastic neural networks are straightforward by using reparameterization tricks (\cite{DBLP:journals/corr/KingmaW13}). In binary SFNN, we use sigmoid as its activation function: 
$p(  {z}_l |   {z}_{l-1}) = \sigma( W_{l-1}   {z}_{l-1} +   {b}_{l-1})$
where $\sigma(.)$ is the (element-wise) sigmoid function, $W_{l-1}$ is the network weights connecting layer $l-1$ to layer $l$, $  {b}_{l-1}$ is a bias vector and $Z_l \in \{0,1\}^{n_l}$. We also make each $Z_{l,i}$ a learnable Bernoulli distribution.

It has been not clear so far how the gradient is computed in stochastic neural network at line $9$ of Algorithm \ref{alg:JointPIB}. The sampling operation in stochastic neural networks precludes the backpropagation in a computation graph. It becomes even more challenging with binary stochastic neural networks as it is not well-defined to compute gradients w.r.t. discrete-valued variables. Fortunately, we can find approximate gradients which has been proved to be efficient in practice: REINFORCE estimator (\cite{DBLP:journals/ml/Williams92,DBLP:journals/corr/BengioLC13}), straight-through estimator (\cite{Hintonlecture}), the generalized EM algorithm (\cite{DBLP:conf/nips/TangS13}), and Raiko (biased) estimator (\cite{DBLP:journals/corr/RaikoBAD14}). Especially, we found the Raiko gradient estimator works best in our specific setting thus deployed it in this application. In the Raiko estimator, the gradient of a bottleneck particle $z_{l,i} \sim p({z}_{l,i} |   {z}_{l-1}) = \sigma (a_i^{(l)})$ is propagated only through the deterministic term $\sigma (a_i^{(l)})$:
$\frac{\partial z_{l,i}}{\partial   {\theta}} \approx \frac{\partial \sigma (a_i^{(l)})}{\partial   {\theta}}.$ 
\section{Experiments}
\label{sec:exp}
\begin{table*}[t]
\vskip 0.15in
  \centering
  \caption{The performance of IMB for classification and adversarial robustness on MNIST and CIFAR10 in comparison with MLE and a partially information-theoretic treatment VIB. IMB utilizes layer-wise compression-relevance trade-offs during the training which outperforms and is more adversarially robust than other models of the same architecture. }
 \label{table:cls_adv}
  
  \begin{tabular}{lcccc}
    \toprule
      \multirow{3}{*}{\textbf{Model}}  &\multicolumn{2}{c}{Classification}   &  \multicolumn{2}{c}{Adv. Robustness (\%)}  \\
 & MNIST & CIFAR10 &\multirow{2}{*}{Targeted} & \multirow{2}{*}{Untargeted} \\
      & (Error \%) & (Accuracy \%) & \\
\midrule
    DET & $1.73  $ & $53.91 $ & $00.00$ & $00.00$  \\ 
    VIB (\cite{DBLP:journals/corr/AlemiFD016}) & $1.45$ & $54.41$ & $83.70$ & $93.10$ \\ 
    SFNN (\cite{DBLP:journals/corr/RaikoBAD14}) & $1.44$ & $55.94 $ & $83.00$ & $95.20$\\ 
    \texttt{GreedyIMB}  & $1.54 $ &  $\textbf{57.61}  $ & $83.21$ & $94.30$ \\ 
    \texttt{JointIMB}   & $\textbf{1.36}$ & $55.62 $ & $ {84.16}$ & $\textbf{96.00}$ \\  
    \bottomrule
  \end{tabular}
  
\end{table*}


In this section, we evaluated IMB with \texttt{GreedyIMB} and \texttt{JointIMB} algorithms on MNIST \cite{LeCun98} and CIFAR10 \cite{cifar10} for classification, learning dynamics and robustness against adversarial attacks. The MNIST dataset consists of 28x28 pixel greyscale images of handwritten digits 0-9, with 60,000 training and 10,000 test examples. The CIFAR10 dataset consists of 60,000 (50,000 for train and 10,000 for test) 32x32 colour images in 10 classes, with 6,000 images per class. 


\subsection{Image classification}



We compared \texttt{JointIMB} and \texttt{GreedyIMB} with other three comparative models which used the same network architecture without any explicit regularizer: (1) Standard deterministic neural network (DET) which simply treats each hidden layer as deterministic; (2) Stochastic Feed-forward Neural Network (SFNN) \cite{DBLP:journals/corr/RaikoBAD14} which is a binary stochastic neural network as in IMB but is trained with the MLE principle; (3) Variational Information Bottleneck (VIB) \cite{DBLP:journals/corr/AlemiFD016} which employs the entire deterministic network as an encoder, adds an extra stochastic layer as a out-of-network bottleneck variable, and is then trained with the IB principle on that single bottleneck layer. The base network architecture in this experiment had two hidden layers with $512$ sigmoid-activated neurons per each layer. The experimental setup details can be found in Appendix D (the supplementary).



The results are shown in Table \ref{table:cls_adv}. Even though we did not optimize for the hyperparameters in IMB, \texttt{JointIMB} already outperforms MLE and VIB on MNIST, \texttt{GreedyMIB} outperforms the other models on CIFAR10. The performance of JointPIB on CIFAR10 is also comparable to MLE. The result suggests a promising effectiveness of explicitly inducing relevant but compressed information into each layer of a neural network via IMBs. 



\subsection{Robustness against adversarial attacks}
We consider here the adversarial robustness of neural networks trained by IMBs. Neural networks are prone to adversarial attacks which disturb the input pixels by small amounts imperceptible to humans \cite{DBLP:journals/corr/SzegedyZSBEGF13,DBLP:conf/cvpr/NguyenYC15}. Adversarial attacks generally fall into two categories: untargeted and targeted attacks. An untargeted adversarial attack $\mathcal{A}$ maps the target model $M$ and an input image $ {x}$ into an adversarially perturbed image $ {x}'$: $\mathcal{A}:(M,  {x}) \rightarrow  {x}'$, and is considered successful if it can fool the model $M( {x}) \neq M( {x}')$. A targeted attack, on the other hand, has an additional target label $l$: $\mathcal{A}:(M,  {x}, l) \rightarrow  {x}'$, and is considered successful if $M( {x}')  = l \neq M( {x})$.

We performed adversarial attacks to the neural networks trained by MLE and IMB, and resorted to the accuracy on adversarially perturbed versions of the test set to rank a model's robustness. In addition, we use the $L_2$ attack method for both targeted and untargeted attacks \cite{DBLP:conf/sp/Carlini017}, which has shown to be most effective attack algorithm with smaller perturbations. Specifically, we attacked the same four comparative models described from the previous experiment on the first $1,000$ samples of the MNIST test set. For the targeted attacks, we targeted each image into the other $9$ labels other than the true label of the image.


The results are also shown in Table \ref{table:cls_adv}. We see that the deterministic model DET is totally fooled by the attacks. It is known that stochasticity in neural networks improves adversarial robustness which is consistent in our experiment as SFNN is significantly more adversarially robust than DET. VIB has compatible adversarial robustness with SFNN even if VIB has ``less stochasticity" than SFNN (VIB has one stochastic layer while all hidden layers of SFNN are stochastic). This is because VIB performance is compensated with IB principle for its stochastic layer. Finally, \texttt{JointIMB} is more adversarially robust than the other models. Explicitly inducing compression and relevance into each layers thus has a potential of being more adversarially robust. 



\subsection{Learning dynamics}
\begin{figure*}[t]
  \centering
  \begin{minipage}[b]{0.5\textwidth}
	\includegraphics[scale=0.07]{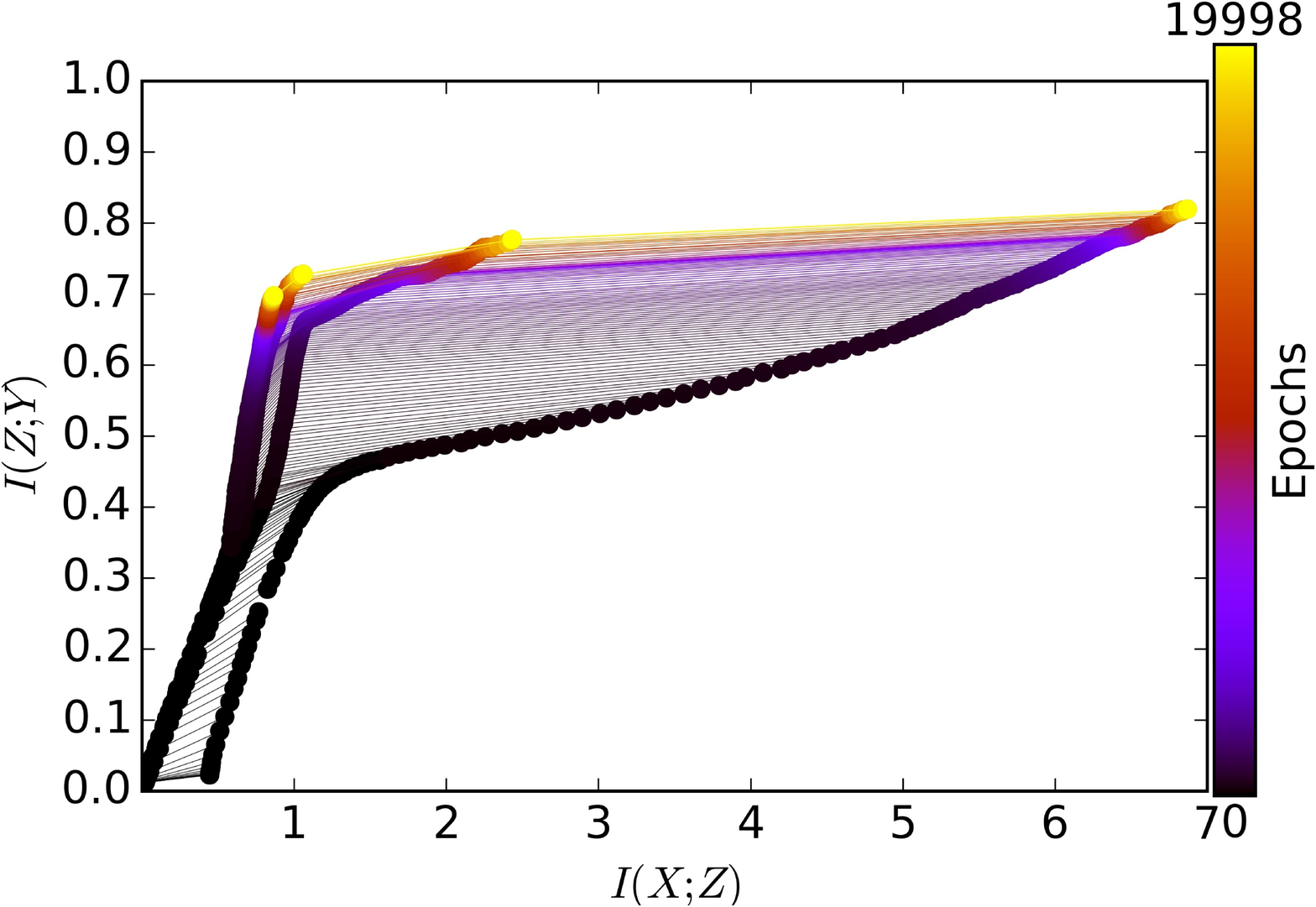}
    \label{fig:dynamics_sfnn}
  \end{minipage}\hfill
  \begin{minipage}[b]{0.5\textwidth}
	\includegraphics[scale=0.07]{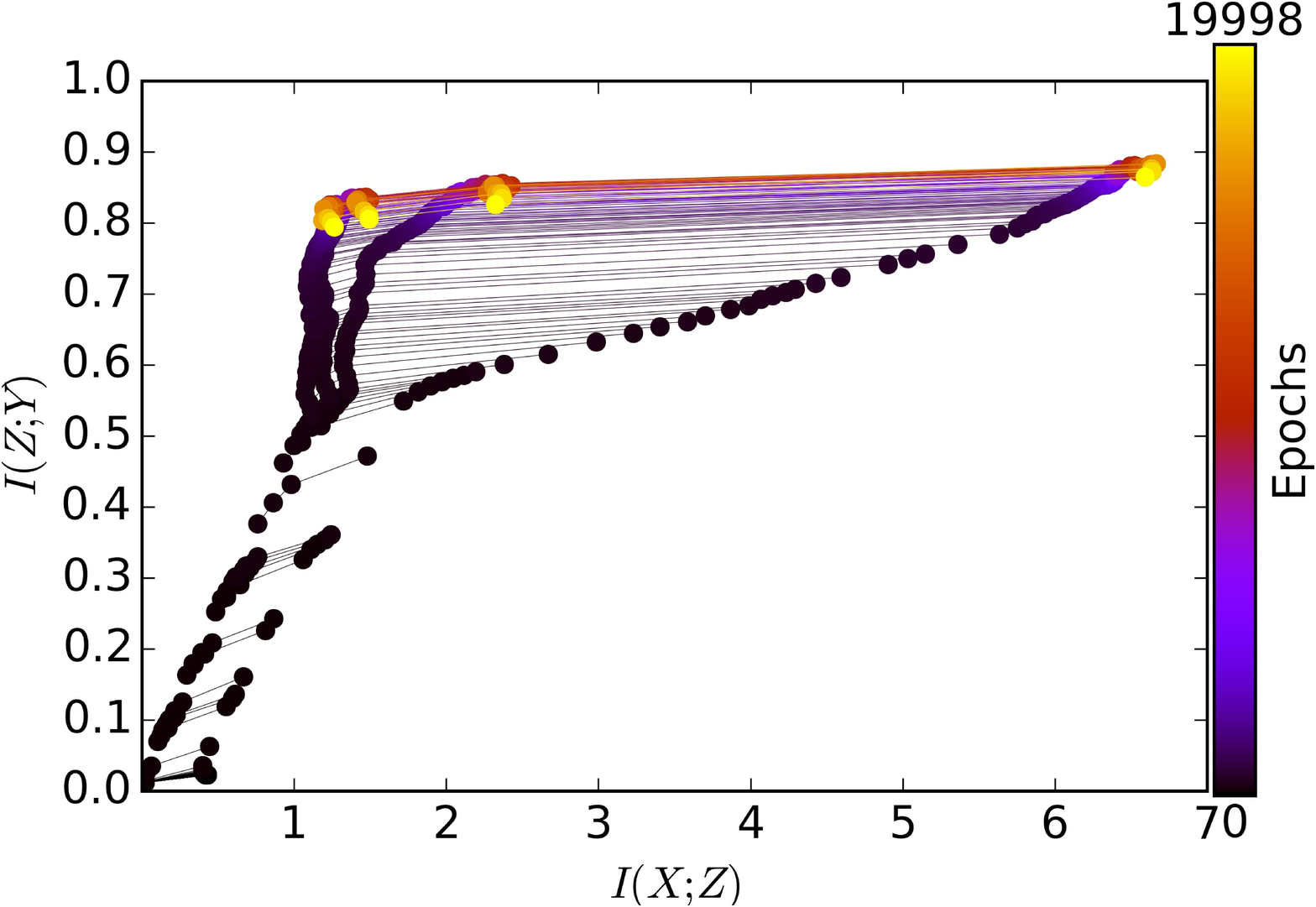}
    \label{fig:dynamics_jointpib}
  \end{minipage}
  \caption{The learning dynamics of SFNN (left) and IMB (right). The color indicates the training epochs while each node in a color in the graph represents $(I(Z_l; X), I(Z_l; Y))$ at the corresponding epoch. Note that at each epoch, $I(Z_{l}; X) \geq I(Z_{l+1}; X), \forall l$ (data-processing inequality). \texttt{JointIMB} jointly encodes relevant information into every layer of DNNs while keeping each layer informatively concise. As compared to MLE, the level of relavent information encoded by \texttt{JointIMB} increases more quickly over training epochs and reaches a higher value.}
  \label{fig:dynamic}
\end{figure*}

To better understand how MIB has modified the information within the layers during the learning process, we visualize the compression and relevance of each layer over the course of training of SFNN and \texttt{JointMIB} (the visualization for \texttt{GreedyPIB} is in the Appendix E (the supplementary)). To simplify our analysis, we considered a binary decision problem where $X$ is $12$ binary inputs making up $2^{12} = 4096$ equally likely input patterns and  $Y$ is a binary variable equally distributed among $4096$ input patterns \cite{DBLP:journals/corr/Shwartz-ZivT17}. The base neural network architecture had 4 hidden layers with widths: 10-8-6-4 neurons. Since the network architecture is small, we could precisely compute the (true) compression $I_x := I(Z_i; X)$ and (true) relevance $I_y := I(Z_i; Y)$ over training epochs. We fixed $\beta_l = \beta = 10^{-4}$ for both \texttt{JointIMB}, trained five different randomly initialized neural networks for each comparative model with SGD up to 20,000 epochs on $80\%$ of the data, and averaged the mutual information.  

Figure \ref{fig:dynamic}  provides a visualization of the learning dynamics of SFNN versus \texttt{JointIMB} on the information plane $(I_x, I_y)$. 
We can observe a common trend in the learning dynamics offered by both MLE (in SFNN model) and \texttt{JointIMB} framework. Both principles allow the network to gradually encode more information about $X$ and the relevant information about $Y$ into the hidden layers at the beginning as $I(Z_i; X)$ and $I(Z_i; Y)$ both increase. Especially, compression does occur in SFNN which is consistent with the result reported in (\cite{DBLP:journals/corr/Shwartz-ZivT17}). 




What distinguishes IMB from MLE is the maximum level of relevance at each layer and the number of epochs to encode the same level of relevance. 
Firstly, \texttt{JointIMB} at $l=1$ needs only about $4.68\%$ of the training epochs to achieve at least the same level of relevance in all layers of SFNN at the final epoch. Secondly, MLE is unable to encode the network layers to reach the maximum level of relevance enabled by IMB (We also trained SFNN up to $100,000$ epochs and observed that the level of relevance of each layer degrades before ever reaching the value of $0.8$ bits.).
We also see that the compression constraints within the IMB framework keep the layer representation from shifting to the the right (in the information plane) during the encoding of relevant information.

The reason that the relevance for IMB increases until some point before decreasing while the relevance for SFNN increases until some point where the value almost stays there (without a decrease) is, we believe, because that IMB can explicitly exploit the information from each layer in a way that is more effective than MLE. The IMB objective can allow the encoding of relevant information into each layer to its optimal information trade-off eventually at some point. After this point if we continue with the training, due to the mismatch between the exact IMB objective and its variational bound (Eq. 6,7, 10 in our draft), the further minimization of the variational bound would decrease $I(Z_l; Y)$ (consequently, in order to make sure $I(Z_l; X) - \beta_l I(Z_l; Y)$ small, $I(Z_l; X)$ also needs to decrease after this point to compensate for the decrease in $I(Z_l; Y)$). In the case of SFNN (trained with MLE), the MLE objective reaches its local minimum before the information of each layer can even reach its optimal information trade-off (if ever). This explains why the relevance and compression of SFNN almost stay the same after some point. This suggests that IMB is better than MLE in terms of exploiting information for each layer during the learning.

I think the claim we made in the draft, which is "Especially, compression does occur in SFNN which is consistent with the result reported in ([28])", is a bit general without much elaboration in the draft. By seeing compression there in Fig. 1.a., I think a more precise wording for that claim of "compression occurs" is that the increase of $I(Z_l; X)$ slows down at some point for deeper layers. Intuitively, in order for the representations $Z_l$ to make sense of the task, the representations should encode enough information about $X$; thus $I(Z_l; X)$ should increase (This is especially true for shallow layers because, due to the Markov chain property, the shallower a layer, the greater its burden of carrying enough information to make sense a task. This might also explain why the compression force, explained below, cannot dominate the force of increasing $I(Z_l; X)$ for shallow layers). However, instead of keeping increasing $I(Z_l; X)$ forever at the same rate, MLE trained with SGD slows down the increase of $I(Z_l; X)$ at some point for deep layers (e.g., Fig. 1.a., for $Z4, Z3, Z2$). The force that slows down the increase of $I(Z_l ;X)$ at some point we would say that compression force takes place. Depending on how strong that force of compression takes place (which in turn I believe depends on which tasks we are demonstrating and the structure of $p(X,Y)$ is used), the slowing down phenomenon might have a stronger effect of bending over the increase of $I(Z_l; X)$. The stronger the compression force, the more severe the bending over of the $I(Z_l; X)$ increase, though in Fig.1.a., the bending effect is not as strong as that present in Schwartz-Ziv and Tishby experiment. The idea that MLE trained with SGD has compression is very interesting because MLE principle is not explicitly set out for compression in its mind, but still has compression effect.  Schwartz-Ziv and Tishby have very well explained the compression effects of MLE trained with SGD; I think, if not mistakenly, most of their explanation is from SGD perspective, not from a learning principle perspective (which is MLE in this case).  I would believe that there is an alternative explanation from a learning principle perspective, and such an explanation would be very interesting. In the meantime, IMB takes over the implicit role of MLE and push it more explicitly: explicitly encourage such compression-relevance tradeoffs in all layers.  In IMB (fig. 1.b), the compression force is stronger (i.e., the bending over effect is stronger than that in MLE) while the relevance is pushed higher in much fewer epoches (represented by the colors). This suggests that MLE does not fully encourage compression-relevance tradeoffs for each layer even though MLE happens to do so implicitly in a limited way. Later down, explicit encouragement of compression-relevance into each layer would bring potential benefits for classification, adversarial learning and multi-modal learning. 

\section{Conclusion}
In this work, we introduce Information Multi-Bottlenecks (IMBs) which provides a principled information-theoretic learning of neural networks. We provide important insights about when IMBs are possible and how they can be approximated in stochastic neural networks. The principled inducing of compression and relevance into each layer of DNNs via IMBs also provides a promising improvement in classification problems and adversarial robustness. In general, IMBs also better exploits the neural representations for learning. We demonstrate the results in MNIST and CIFAR10. 

\newpage


\bibliographystyle{plain}
\bibliography{refs.bib}

\section*{Appendix}
\subsection*{Appendix A. Proof of Theorem \ref{thm:conflicting_optimality}}

\begin{proof}[{Proof} of Theorem \ref{thm:conflicting_optimality}]
We present a detailed proof for Theorem \ref{thm:conflicting_optimality} which uses the contradiction technique and  three following lemmas. 

\begin{lem}
\label{lemma:lemma1}
Given $Y \rightarrow X \rightarrow Z_1 \rightarrow Z_2$, we have 
\begin{align}
\label{eq:lemma1_eq1}
I(Z_2; X) &= I(Z_1;X) - I(Z_1;X | Z_2) \\
\label{eq:lemma1_eq2}
I(Z_2; Y) &= I(Z_1;Y) - I(Z_1;Y | Z_2)
\end{align}
\end{lem}
\begin{proof}[\textbf{Proof}]
If follows from \cite{Cover06} that
$I(X; Z_1;Z_2) = I(X; Z_2) {+} I(X; Z_1 | Z_2) \allowbreak{=} I(X; Z_1) {+} I(X; Z_2 | Z_1)$, but $I(X; Z_2 | Z_1) {=} 0$ since $X \not\perp Z_2 | Z_1$., hence Eq. (\ref{eq:lemma1_eq1}). The proof for Eq.  (\ref{eq:lemma1_eq2}) is similar by replacing variable $X$ with variable $Y$. 
\end{proof}
\begin{lem}
\label{lemma:lemma2}
Given $Y \rightarrow X \rightarrow Z_1 \rightarrow Z_2, 0 < \beta_2 < \infty$ and $H(X|Y) > 0$, let us define the {conditional Information} {Bottleneck} objective: 
\begin{align}
    \mathcal{L}^c := \mathcal{L}^c[p(z_2|z_1), p(z_1|x)] := I(Z_1;X | Z_2) - \beta_2 I(Z_1;Y | Z_2).
\end{align}
If $Z_1$ and $Z_2$ do not satisfy either of the following conditions: 
\begin{enumerate}
\item $Z_2$ is a \textit{sufficient statistic} of $Z_1$ for $X$ and $Y$ (i.e., $Y \rightarrow X \rightarrow Z_2 \rightarrow Z_1$),     \item $Z_2$ is independent of $Z_1$.
\end{enumerate}
Then, $\partial \mathcal{L}^c/ \partial p(z_1|x)$ depends on $\{p(z_2 | z_1)\}$. Informally, if the conditional variable $Z_2$ in the conditional Information Bottleneck objective $\mathcal{L}^c$ is not an ``trivial" transform of the bottleneck variable $Z_1$, $Z_2$ induces a non-trivial topology into the conditional Information Bottleneck objective. 
\end{lem}
\begin{proof}[\textbf{Proof}]
 By the definition of the conditional mutual information
 \begin{align*}
     I(Z_1; X | Z_2) = \int \int \int p(x, z_1, z_2) \log \frac{p(z_1, x | z_2)}{p(z_1 | z_2) p(x|z_2)} dz_2 dz_1 dx
 \end{align*}
 $I(Z_1; X | Z_2)$ depends on $p(x, z_1, z_2)$ as long as the presence of $Z_2$ in the conditional Information Bottleneck objective does not vanish (we will discuss the conditions for $Z_2$ to vanish in the final part of this proof). Note that due to the Markov chain $X \rightarrow Z_1 \rightarrow Z_2$, we have $p(x, z_1, z_2) = p(x) p(z_1 |x) p(z_2 | z_1)$. Thus, $\partial I(Z_1; X | Z_2) / \partial p(z_1 | x)$ depends on $p(z_2 | z_1)$ as long as $Z_2$ does not vanish in the objective. The same result is applied to $\partial I(Z_1; Y | Z_2) / \partial p(z_1 | x)$; hence $\partial \mathcal{L}^c/ \partial p(z_1|x)$ depends on $\{p(z_2 | z_1)\}$ (note that $H(X|Y) > 0$ prevents the collapse of $\mathcal{L}^c$ when summing two mutual information) if $Z_2$ does not vanish in the objective. 
 
 Now we discuss the vanishing condition for $Z_2$ in the objective. Note that it follows from Lemma \ref{lemma:lemma1} that: 
\begin{align*}
     0 \leq I(Z_1; X | Z_2) \leq I(Z_1; X) \\
     0 \leq I(Z_1; Y | Z_2) \leq I(Z_1; Y)
 \end{align*}
 
 It is easy to see that $Z_2$ vanishes in the conditional Information Bottleneck objective $\mathcal{L}^c$ iff each of the mutual information in $\mathcal{L}^c$  does not depend on $Z_2$ iff the equality conditions for the inequalities immediately above are met. Note that if $I(Z_1; X | Z_2) = 0$, then then we have $Y \rightarrow X \rightarrow Z_2 \rightarrow Z_1$ (i.e., $Z_2$ is a sufficient statistic for $X$ and $Y$). This implies that $I(Z_1; Y | Z_2) = 0$. Similarly, $I(Z_1; X | Z_2) = I(Z_1; X)$ implies that $Z_2$ is independent of $Z_1$ which in turn implies that $I(Z_1; Y | Z_2) = I(Z_1; Y)$.
\end{proof}

Now we prove Theorem (\ref{thm:conflicting_optimality}) by contradiction and three lemmas above. 

$(\Rightarrow)$ First we prove that if $Z_1$ and $Z_2$ satisfy neither condition (a) nor (b) in the theorem, the constrained minimization problems are conflicting. Assume, by contradiction, that there exists a solution that minimizes both $\mathcal{L}_1$ and $\mathcal{L}_2$ simultaneously, i.e., $\exists p( {z}_1| {x}), p( {z}_2 |  {z}_1) $ s.t. $\mathcal{L}_1$ has a minimum at $\{p( {z}_1| {x})\}$ and $\mathcal{L}_2$ has a minimum at $\{p( {z}_1| {x}), p( {z}_2 |  {z}_1)\}$. Note that $\{p( {z}_1| {x})\}$ and $\{p( {z}_2| z_1)\}$ are independent variables for the optimizations. By introducing Lagrangian multipliers $\lambda_1(x)$ and $\lambda_2(x)$ for the constraint $\int p(z_1|x) dz_1 = 1$ of $\mathcal{L}_1$ and $\mathcal{L}_2$, respectively, we obtain: 
\begin{align}
    \frac{\partial L_1}{\partial p(z_1 | x)} = 0 \\
    \frac{\partial L_2}{\partial p(z_1 | x)} = 0
\end{align}
where 
\begin{align}
    L_1 := I(Z_1; X) - \beta_1 I(Z_1; Y) - \int\int\lambda_1(x) p(z_1 | x) dz_1 dx \\
    L_2 := I(Z_2; X) - \beta_2 I(Z_2; Y) - \int\int\lambda_2(x) p(z_1 | x) dz_1 dx. 
\end{align}
It follows from Lemma \ref{lemma:lemma1} that:
\begin{align}
    \label{eq:L2_L1}
    L_2 - L_1 = (\beta_1 - \beta_2) I(Z_1; Y) - \mathcal{L}^c -  \int\int (\lambda_2(x) - \lambda_1(x)) p(z_1 | x) dz_1 dx. 
\end{align}
where $\mathcal{L}^c = I(Z_1;X | Z_2) - \beta_2 I(Z_1;Y | Z_2)$ (defined as in Lemma \ref{lemma:lemma2}). Now take the derivative w.r.t $p(z_1 | x)$ both sides of Eq. \ref{eq:L2_L1} with notice that $\partial L_1 / \partial p(z_1 | x) = \partial L_2 / \partial p(z_1 | x) = 0$, we have: 
\begin{align}
    \label{eq:contradict_eq}
    \frac{\partial \mathcal{L}^c}{\partial p(z_1 | x)} = (\beta_1 - \beta_2) \frac{\partial I(Z_1; Y) }{\partial p(z_1 | x)} + \lambda_1(x) - \lambda_2(x). 
\end{align}
Notice that the left hand side of Eq. \ref{eq:contradict_eq} strictly depends on $\{p(z_2 | z_1)\}$ (Lemma \ref{lemma:lemma2}) while the right hand side is independent of $\{p(z_2 | z_1)\}$. This contradiction implies that the initial existence assumption is invalid; thus implies the conclusion in Theorem \ref{thm:conflicting_optimality}. 

($\Leftarrow$) The the direction is obvious. When $Z_1$ and $Z_2$ satisfy condition (a) (i.e., $I(Z_2; X) = I(Z_1; X)$ and $I(Z_2; Y) = I(Z_1; Y)$) or (b) (i.e., $I(Z_2; X) = 0$ and $I(Z_2; Y) = 0$) in the theorem, there are effectively only one optimization problem for $\mathcal{L}_1$, and this reduces into the original Information Bottleneck (with single bottleneck) (\cite{Tishby99}). After solving for $Z_1$ from the Information Bottleneck optimization, we can construct $Z_2$ as a sufficient statistic of $Z_1$.  
\end{proof}

\subsection*{Appendix B. Proof of Theorem 3.2}
\begin{proof}[Proof of Theorem 2]
The first claim of the theorem immediately follows from the definition of VCR: 
\begin{align}
\tilde{H}(Y | Z_0) &= -\mathbb{E}_{(X, Y)_D} \mathbb{E}_{Z_0|X}  \log p(\hat{Y} | Z_0) \nonumber \\ 
&= -\mathbb{E}_{(X, Y)_D} \log p(\hat{Y} | X).
\end{align}
For the second claim of the theorem, it follows from the Markov assumption in Equation (1) and from Jensen's inequality, respectively, that: 
\begin{align}
p(\ {\hat{y}}|\ {z}) =   p(\ {\hat{y}}|\ {z}_L, \ {z}_{L-1}, ..., \ {z}_1) = p(\ {\hat{y}}|\ {z}_L)
\end{align}
\begin{align}
\int p(\ {z}|\ {x}) \log p (\ {\hat{y}}|\ {z}) d\ {z} &\leq \log \left( \int p(\ {z}|\ {x})p (\ {\hat{y}}|\ {z}) d\ {z} \right) \nonumber \\
&= \log p(\ {\hat{y}}|\ {x})
\end{align}
Thus, we have 
\begin{align}
- \mathbb{E}_{(X,Y)_D} \log p(\hat{Y}|X) &\leq  - \mathbb{E}_{(X,Y)_D} \int p(\ {z}|X) \log p (\hat{Y}|\ {z}) d\ {z} \nonumber \\
&\!\!\!\!\!\!\!\!\!\!\!\!\!=  -\mathbb{E}_{(X, Y)_D} \mathbb{E}_{Z|X}  \log p(\hat{Y} | Z) \nonumber\\ 
&\!\!\!\!\!\!\!\!\!\!\!\!\!= \tilde{H}(Y | Z) \\
&\!\!\!\!\!\!\!\!\!\!\!\!\!=  -\mathbb{E}_{(X, Y)_D} \mathbb{E}_{Z_L|Z_{L-1}, ..., Z_1|Z_0}  \log p(\hat{Y} | Z_L) \nonumber\\ 
&\!\!\!\!\!\!\!\!\!\!\!\!\!= -\mathbb{E}_{(X, Y)_D} \mathbb{E}_{Z_L|X}  \log p(\hat{Y} | Z_L) \nonumber\\ 
&\!\!\!\!\!\!\!\!\!\!\!\!\!= \tilde{H}(Y | Z_L),
\end{align}

\subsection*{Appendix C. Detailed derivations of Approximate Compression}
\begin{align}
I(Z_l; X) &\leq I(Z_l; Z_{l-1}) = \int p( {z}_{l} |  {z}_{l-1} ) p( {z}_{l-1}) \log \frac{p( {z}_l |  {z}_{l-1})}{p( {z}_l)} d  {z}_l d  {z}_{l-1} \nonumber \\
&\leq \int p( {z}_{l} |  {z}_{l-1} ) p( {z}_{l-1}) \log \frac{p( {z}_l |  {z}_{l-1})}{r( {z}_l)} d  {z}_l d  {z}_{l-1} 
 =  \mathbb{E}_{Z_{l-1}}  D_{KL} \left[ p(Z_{l} | Z_{l-1}) || r(Z_{l}) \right] \nonumber \\
&=  \mathbb{E}_{Z_{l-1}} \sum_{i=1}^{n_l} D_{KL} \left[ p(Z_{l,i} | Z_{l-1}) || r(Z_{l,i}) \right] =: \tilde{I}(Z_l; Z_{l-1})
\end{align}

\subsubsection*{Appendix D. Detailed experimental description in classification experiment}
Adopted from the common practice, we used the last 10,000 images of the training set as a validation (holdout) set for tuning hyper-parameters. We then retrained models from scratch in the full training set with the best validated configuration. We trained each of the five models with the same set of 5 different initializations and reported the average results over the set. For the stochastic models (all except DET), we drew $M = 32$ samples per stochastic layer during both training and inference, and performed inference $10$ times at test time to report the mean of classification errors for MNIST and classification accuracy for CIFAR10. For \texttt{JointIMB} and \texttt{GreedyIMB}, we set $\gamma_l = 1$ (in \texttt{JointIMB} only) and $\beta_l = \beta, \forall 1 \leq l \leq L$, tuned $\beta$ on a linear log scale $\beta \in \{10^{-i}: 1 \leq i \leq 10 \}$. We found $\beta = 10^{-4}$ worked best for both models. For VIB, we found that $\beta=10^{-3}$ and $\beta=10^{-4}$ worked best on MNIST and CIFAR10, respectively. We trained the models on MNIST with Adadelta optimization (\cite{DBLP:journals/corr/abs-1212-5701}) and on CIFAR10 with Adagrad optimization (\cite{DBLP:journals/jmlr/DuchiHS11}) (except for VIB we used Adam optimization (\cite{DBLP:journals/corr/KingmaB14}) as we found that they worked best in the validation set. 

\subsubsection*{Appendix E. Learning dynamics of \texttt{GreedyIMB}}
We further present the visualization of the learning dynamics of \texttt{GreedyIMB} in Figure \ref{fig:GreedyPIB_dyn}.  \texttt{GreedyIMB} at $l=1$ needs only about $17.95\%$ of the training epochs to achieve at least the same level of relevance in all layers of SFNN at the final epoch. Recall that in \texttt{GreedyIMB} at $l{=}1$ the PIB principle is applied to the first hidden layer only. The layer representation at the final epoch gradually shifts to the left (i.e., more compressed) while not degrading the relevance over the greedy training from layer $1$ to layer $4$ in Figure \ref{fig:GreedyPIB_dyn}.
We also see that the compression constraints within the IMB framework keep the layer representation from shifting to the the right (in the information plane) during the encoding of relevant information. 
\begin{figure*}[t]
  \centering
  \begin{minipage}[b]{0.5\textwidth}
	\includegraphics[scale=0.07]{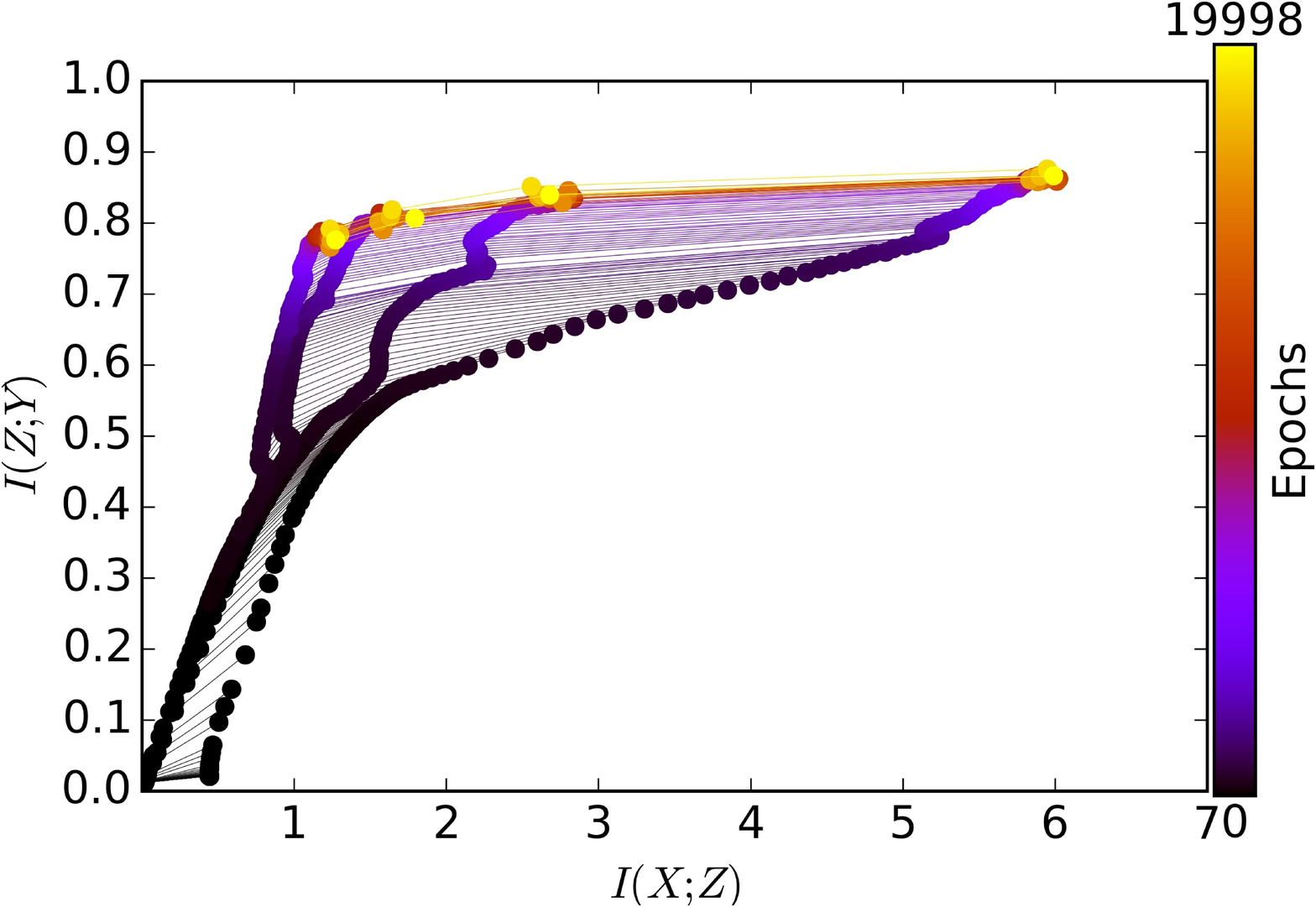}
    \label{fig:dynamics_greedy_l1}
  \end{minipage}\hfill
  \begin{minipage}[b]{0.5\textwidth}
	\includegraphics[scale=0.07]{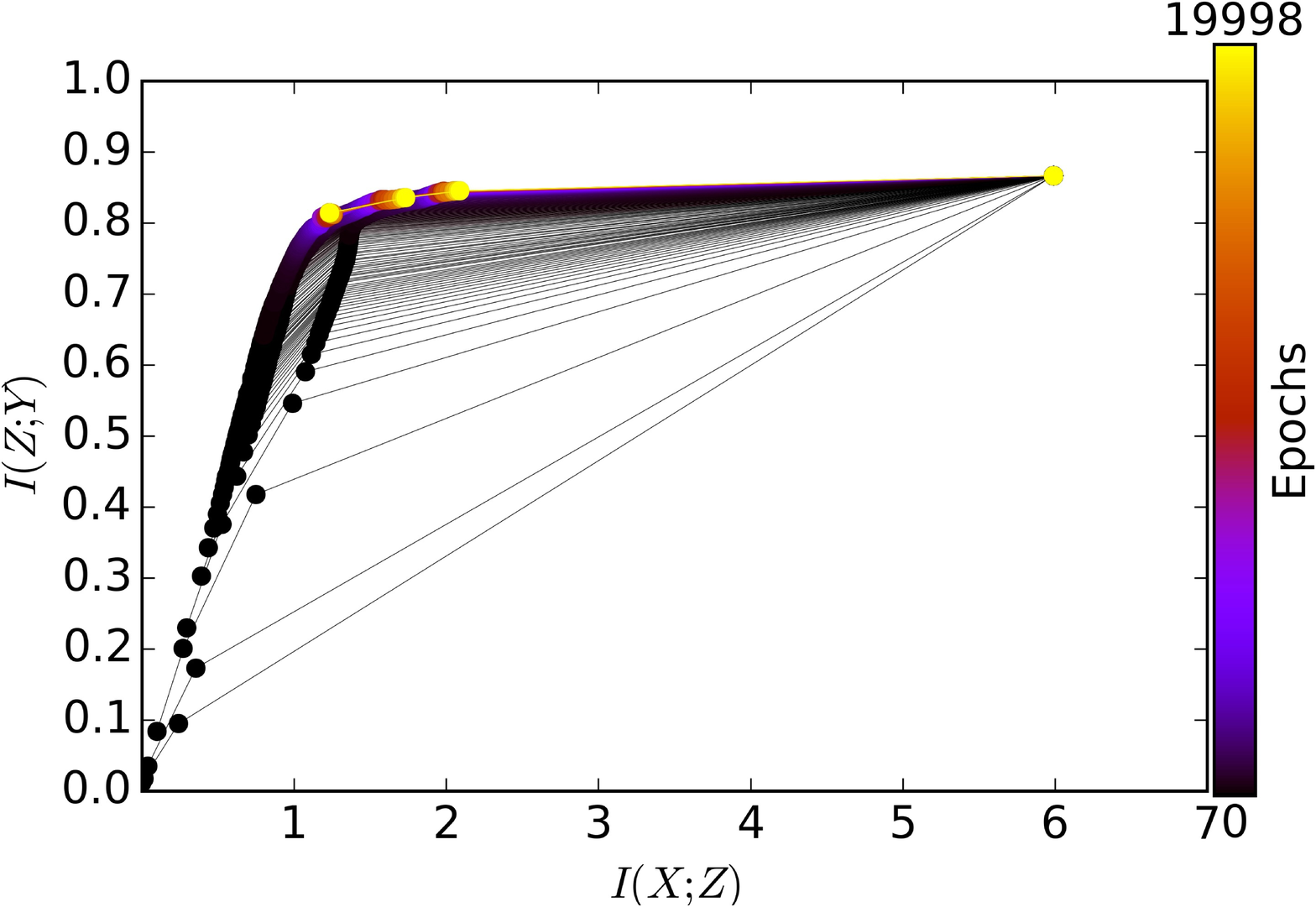}
     \label{fig:dynamics_greedy_l2}
  \end{minipage}
  \begin{minipage}[b]{0.5\textwidth}
	\includegraphics[scale=0.07]{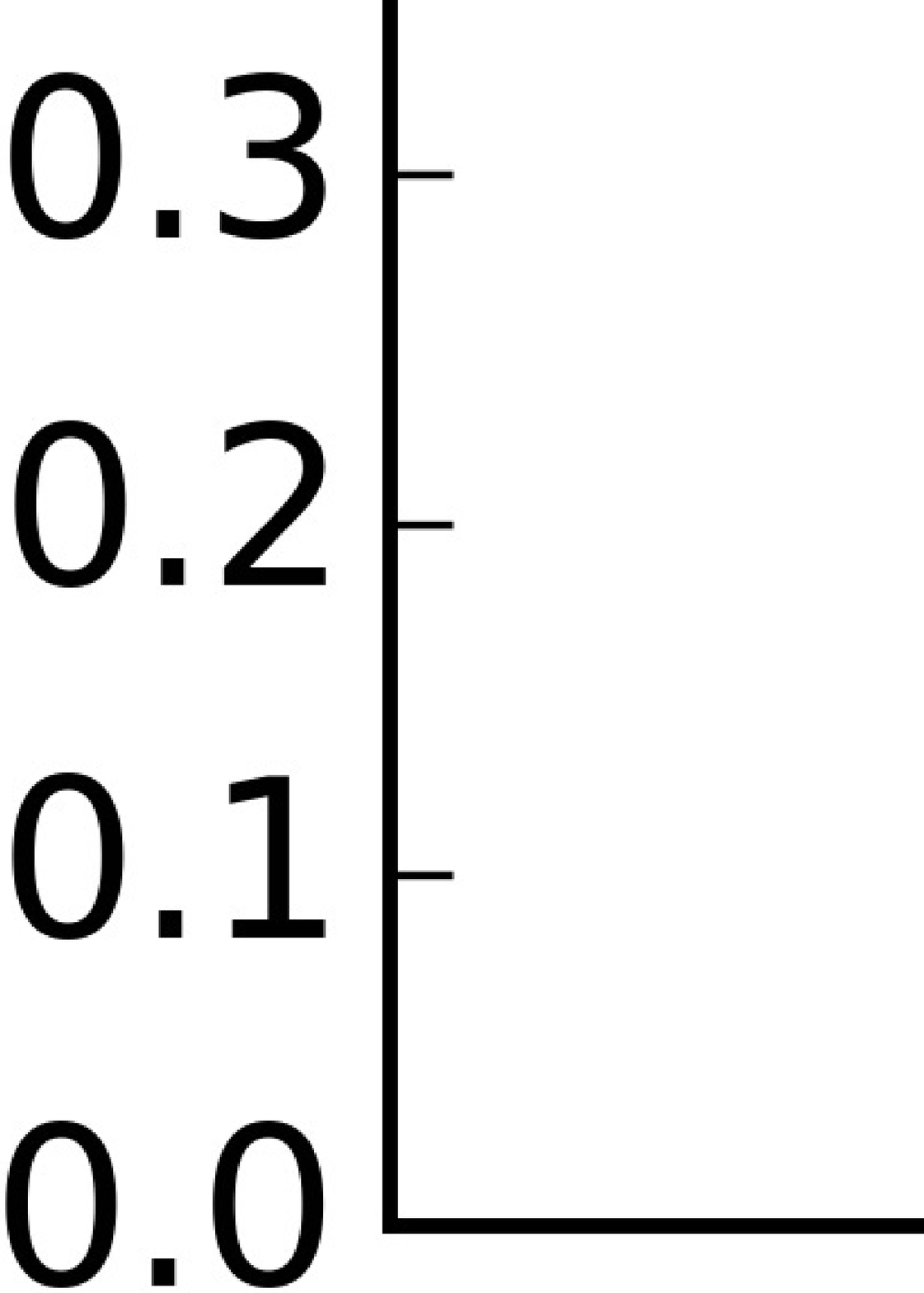}
     \label{fig:dynamics_greedy_l3}
  \end{minipage}\hfill 
  \begin{minipage}[b]{0.5\textwidth}
	\includegraphics[scale=0.07]{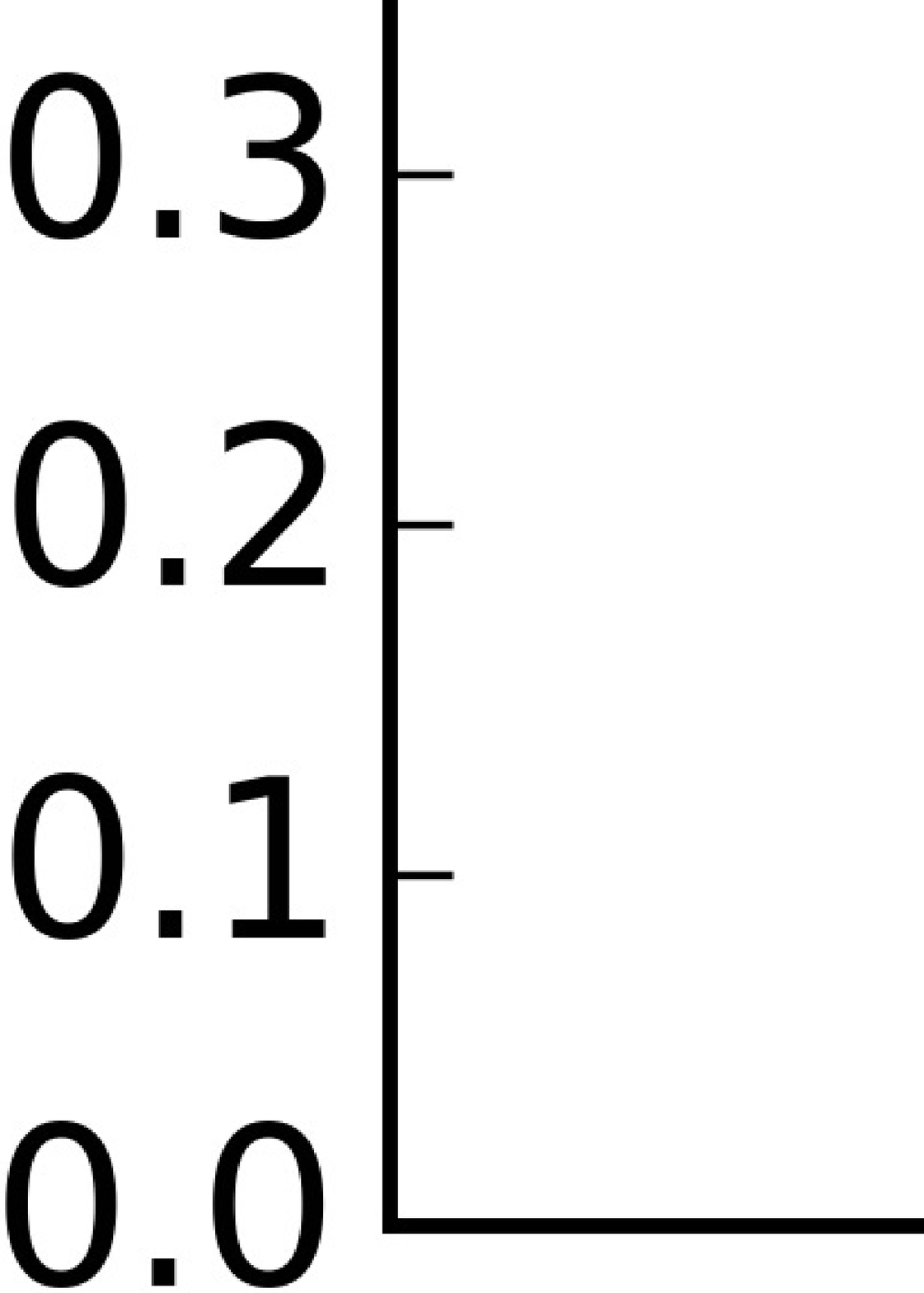}
     \label{fig:dynamics_greedy_l4}
  \end{minipage}
  \caption{From left to right, top to bottom represent \texttt{GreedyIMB}'s encoding of relevant information into layer $1 \leq l \leq 4$, respectively. \texttt{GreedyIMB} greedily encodes relevant information into each layer given the information optimality of the previous layers. \texttt{GreedyIMB} also achieves a significantly higher level of relevant information at each layer compared to MLE.}
  \label{fig:GreedyPIB_dyn}
\end{figure*}

\subsubsection*{Appendix F. VCR decomposition for a multivariate target variable}
We will prove that the VCR of level $l$ for a multivariate variable ${y}$ can be decomposed as the sum of the VCRs of each of its vector elements. Indeed, consider $Y = (Y_1,..., Y_n)$. It follows from the fact that the neurons within a layer are conditionally independent given the previous layer that we have:
\begin{align*}
\tilde{H}(Y|Z_l) &= -\mathbb{E}_{(X,Y_1, ..., Y_n)_D} \mathbb{E}_{Z_l|X} \log  p(\hat{Y}_1,..., \hat{Y}_n|Z_l) \\ 
&\!\!\!\!\!\!\!\!\!\!\!= -\mathbb{E}_x \mathbb{E}_{Y_1,..., Y_n | X} \mathbb{E}_{Z_l|X} \sum_{i=1}^n \log p(\hat{Y}_i | Z_l)  \\ 
&\!\!\!\!\!\!\!\!\!\!\!= \sum_{i=1}^{n} -\mathbb{E}_{X} \mathbb{E}_{Y_i| X} \mathbb{E}_{Z_l|X} \log p(\hat{Y}_i | Z_l) \\ 
&\!\!\!\!\!\!\!\!\!\!\!= \sum_{i=1}^n \tilde{H}(Y_i|Z_l),
\end{align*}
implying the claim about the decomposibility of VCR for a multivariate target variable.
\end{proof}

\end{document}